\documentclass[english,a4paper,12pt]{article}

\PassOptionsToPackage{pdfpagelabels=false}{hyperref}

\usepackage{amsmath, amsxtra, amsfonts, amssymb, amstext, mathtools}
\usepackage{amsthm}
\usepackage{booktabs}
\usepackage{nicefrac}
\usepackage{xspace}
\usepackage[noadjust]{cite}
\usepackage{url}
\usepackage{graphics}
\usepackage[usenames,dvipsnames]{xcolor}
\usepackage[colorlinks]{hyperref}
\definecolor{linkblue}{rgb}{0.1,0.1,0.8}
\hypersetup{colorlinks=true,linkcolor=linkblue,filecolor=linkblue,urlcolor=linkblue,citecolor=linkblue}
\usepackage[algo2e,ruled,vlined,linesnumbered]{algorithm2e}
\usepackage{wrapfig}

\allowdisplaybreaks[4]
\newtheorem{theorem}{Theorem}
\newtheorem{lemma}[theorem]{Lemma}

\newcommand{\oea}{\mbox{$(1 + 1)$~EA}\xspace}

\newcommand{\OM}{\textsc{OneMax}\xspace}
\newcommand{\onemax}{\OM}
\newcommand{\LO}{\textsc{Leading\-Ones}\xspace}
\newcommand{\leadingones}{\LO}
\newcommand{\binval}{\textsc{BinVal}\xspace}

\DeclareMathOperator{\jump}{\textsc{jump}}

\DeclareMathOperator{\Sample}{Sample}
\DeclareMathOperator{\minmax}{minmax}
\DeclareMathOperator{\poly}{poly}
\newcommand{\Ymax}{Y_{\max}}

\newcommand{\R}{\ensuremath{\mathbb{R}}}

\newcommand{\N}{\ensuremath{\mathbb{N}}} 

\newcommand{\calF}{\ensuremath{\mathcal{F}}}

\DeclareMathOperator{\Bin}{Bin}

\newcommand{\Var}{\mathrm{Var}\xspace} 
\newcommand{\eps}{\varepsilon}

\newcommand{\assign}{\leftarrow}

\begin{document}
{\sloppy

\title{An Exponential Lower Bound for the Runtime of the cGA on Jump Functions\thanks{Notice required by the publisher of the final version of this work: (c) by the author. This is the author's version of the submitted version of this work. It is posted here for your personal use. Not for redistribution. The definitive version will be published in ``Benjamin Doerr. 2019. An Exponential Lower Bound for the Runtime of the Compact Genetic Algorithm on Jump Functions. In Foundations of Genetic Algorithms XV (FOGA '19), August 27--29, 2019, Potsdam, Germany. ACM, New York, NY, USA, 9 pages. https://doi.org/10.1145/3299904.3340304.''}}

\author{Benjamin Doerr\\ \'Ecole Polytechnique\\ CNRS\\ Laboratoire d'Informatique (LIX)\\ Palaiseau\\ France}

\maketitle
 
\begin{abstract}
  In the first runtime analysis of an estimation-of-distribution algorithm (EDA) on the multi-modal jump function class, Hasen\"ohrl and Sutton (GECCO 2018) proved that the runtime of the compact genetic algorithm with suitable parameter choice on jump functions with high probability is at most polynomial (in the dimension) if the jump size is at most logarithmic (in the dimension), and is at most exponential in the jump size if the jump size is super-logarithmic. The exponential runtime guarantee was achieved with a hypothetical population size that is also exponential in the jump size. Consequently, this setting cannot lead to a better runtime.
  
  In this work, we show that any choice of the hypothetical population size leads to a runtime that, with high probability, is at least exponential in the jump size. This result might be the first non-trivial exponential lower bound for EDAs that holds for arbitrary parameter settings.
\end{abstract}


\section{Introduction}

Due to the inherent highly complex stochastic processes, the mathematical analysis of estimation-of-distribution algorithms (EDAs) is still in its early childhood. Whereas for classic evolutionary algorithms many deep analysis exist, see, e.g.,~\cite{NeumannW10,AugerD11,Jansen13}, for EDAs even some of the most basic problems are not fully understood, such as the runtime of the compact genetic algorithm (cGA) on the \onemax benchmark function~\cite{Droste06,SudholtW16,LenglerSW18}. We direct the reader to the recent survey~\cite{KrejcaW18} for a complete picture of the state of the art in mathematical analyses of EDAs.  

Given this state of the art, it is not surprising that the first runtime analysis of an EDA on a multi-unimodal objective function appeared only very recently. In their GECCO 2018 paper, Hasen\"ohrl and Sutton~\cite{HasenohrlS18} analyze the optimization time of the cGA on the jump function class. Jump functions are simple unimodal functions except that they have a valley of low fitness of scalable size $k$ around the global optimum. Hasen\"ohrl and Sutton show~\cite[Theorem~3.3]{HasenohrlS18} that, for a sufficiently large constant $C$ and any constant $\eps > 0$, the cGA with hypothetical population size at least ${\mu \ge \max\{C n e^{4k}, n^{3.5+\eps}\}}$ with probability $1 - o(1)$ finds the optimum of any jump function with jump size at most $k = o(n)$ in $O(\mu n^{1.5} \log n + e^{4k})$ generations (which is also the number of fitness evaluations, since the cGA evaluates two search points in each iteration).

We note that this runtime guarantee is polynomial (in the problem dimension $n$) when $k = O(\log n)$ and exponential (in $k$) otherwise. Both parts of the result are remarkable when recalling that most classical evolutionary algorithms need time $\Omega(n^k)$. 

For the polynomial part,  the upper bound of order $\mu n^{1.5} \log n$, which is $n^{5+\eps}$ when choosing $\mu$ optimally, was for $k < \frac 1 {20} \ln n$ recently~\cite{Doerr19} improved to $O(\mu \sqrt{n})$ for all $\mu = \Omega(\sqrt n \log n) \cap \poly(n)$, which is $O(n \log n)$ for the optimal choice $\mu = \Theta(\sqrt n \log n)$. Note that $n \log n$ is the asymptotic runtime of many evolutionary algorithms, including the cGA with good parameter choices~\cite{Droste06,SudholtW16,LenglerSW18}, on the simple unimodal \onemax problem. Hence this result shows that the cGA does not suffer significantly from the valley of low fitness around the optimum which is characteristic for jump functions (as long as this valley is not too wide, that is, $k < \frac 1 {20} \ln n$). If we are willing to believe that $\Omega(n \log n)$ is also a lower bound for the runtime of the cGA on these jump functions (which given the results for \onemax appears very plausible, but which seems hard to prove, see Section~\ref{sec:nlogn}), then the result in~\cite{Doerr19} determines the precise asymptotic runtime of the cGA with optimal parameter choice for $k < \frac 1 {20} \ln n$.

What is left open by these two previous works is how good the exponential upper bound (for $k$ super-logarithmic in $n$) is. Since Hasen\"ohrl and Sutton prove their exponential runtime guarantee only for a hypothetical population size $\mu = \Omega(n e^{4k})$, it is clear that they, in this setting, cannot have a sub-exponential runtime (for the sake of completeness, we shall make this elementary argument precise in Lemma~\ref{lem:elem}). So the question remains if, by choosing a smaller hypothetical population size, one could have obtained a better runtime guarantee. 

\textbf{Our main result} is a negative answer to this question. In Theorem~\ref{thm:main} we show that, regardless of the hypothetical population size, the runtime of the cGA on a jump function with jump size $k$ is at least exponential in $k$ with high probability. Interestingly, not only our result is a uniform lower bound independent of the hypothetical population size, but our proof is also ``uniform'' in the sense that it does neither need case distinctions w.r.t.\ the hypothetical population size nor w.r.t.\ different reasons for the lower bound. Here we recall that the existing runtime analyses, see, e.g., again~\cite{Droste06,SudholtW16,LenglerSW18} find two reasons why an EDA can be inefficient. (i) The hypothetical population size is large and consequently it takes long to move the frequencies into the direction of the optimum. (ii) The hypothetical population size is small and thus, in the absence of a strong fitness signal, the random walk of the frequencies brings some frequencies close to the boundaries of the frequency spectrum; from there they are hard to move back into the game. 

We avoid such potentially tedious case distinctions via an elegant drift argument on the sum of the frequencies. Ignoring some technicalities here, we show that, regardless of the hypothetical population size, the frequency sum overshoots a value of $n - \frac 14 k$ only after an expected number of $2^{\Omega(k)}$ iterations. However, in an iterations where the frequency sum is below $n - \frac 14 k$, the optimum is sampled only with probability $2^{-\Omega(k)}$. These two results give our main result.

As a \textbf{side result}, we show in Section~\ref{sec:nlogn} that a result like ``\onemax is an easiest function with a unique global optimum for the cGA'', if true at all, cannot be proven along the same lines as the corresponding results for many mutation-based algorithms. This in particular explains why we and the previous works on this topic have not shown an $\Omega(n \log n)$ lower bound for the runtime of the cGA on jump functions.

\section{Preliminaries}

\subsection{The Compact Genetic Algorithm}

The \emph{compact genetic algorithm} (cGA) is an estimation-of-distribution algorithm (EDA) proposed by Harik, Lobo, and Goldberg~\cite{HarikLG99} for the maximization of pseudo-Boolean functions $\calF : \{0,1\}^n \to \R$. Being a univariate EDA, it develops a probabilistic model described by a frequency vector $f \in [0,1]^n$. This frequency vector describes a probability distribution on the search space $\{0,1\}^n$. If $X = (X_1, \dots, X_n) \in \{0,1\}^n$ is a search point sampled according to this distribution---we write \[X \sim \Sample(f)\] to indicate this---then we have $\Pr[X_i = 1] = f_i$ independently for all $i \in [1..n] \coloneqq \{1, \dots, n\}$. In other words, the probability that $X$ equals some fixed search point $y$ is 
\[\Pr[X = y] = \prod_{i : y_i = 1} f_i \prod_{i : y_i = 0} (1 - f_i).\]

In each iteration, the cGA updates this probabilistic model as follows. It samples two search points $x^1, x^2 \sim \Sample(f)$, computes the fitness of both, and defines $(y^1,y^2) = (x^1,x^2)$ when $x^1$ is at least as fit as $x^2$ and $(y^1,y^2) = (x^2,x^1)$ otherwise. Consequently, $y^1$ is the rather better search point of the two. We then define a preliminary model by $f' \coloneqq f + \frac 1 \mu (y^1 - y^2)$. This definition ensures that, when $y^1$ and $y^2$ differ in some bit position $i$, the $i$-th preliminary frequency moves by a step of $\frac 1 \mu$ into the direction of $y^1_i$, which we hope to be the right direction since $y^1$ is the better of the two search points. The \emph{hypothetical populations size}~$\mu$ is used to control how strong this update is. 

To avoid a premature convergence, we ensure that the new frequency vector is in $[\frac 1n, 1 - \frac 1n]^n$ by capping too small or too large values at the corresponding boundaries. More precisely, for all $\ell \le u$ and all $r \in \R$ we define 
\[
\minmax(\ell,r,u) \coloneqq \max\{\ell,\min\{r,u\}\} = \begin{cases} 
\ell & \mbox{if $r < \ell$}\\
r & \mbox{if $r \in [\ell,u]$}\\
u & \mbox{if $r > u$}
\end{cases}
\] 
and we lift this notation to vectors by reading it component-wise. Now the new frequency vector is $\minmax(\frac 1n \mathbf{1}_n, f', (1 - \frac 1n) \mathbf{1}_n)$.

This iterative frequency development is pursued until some termination criterion is met. Since we aim at analyzing the time (number of iterations) it takes to sample the optimal solution (this is what we call the \emph{runtime} of the cGA), we do not specify a termination criterion and pretend that the algorithm runs forever.

The pseudo-code for the cGA is given in Algorithm~\ref{alg:cga}. We shall use the notation given there frequently in our proofs. For the frequency vector $f_t$ obtained at the end of iteration $t$, we denote its $i$-th component by $f_{i,t}$ or, when there is no risk of ambiguity, by $f_{it}$.
	
\begin{algorithm2e}%
	$t \assign 0$\;
	$f_t = (\frac 12, \dots, \frac 12) \in [0,1]^n$\;
	\Repeat{forever}{
    $x^1 \assign \Sample(f_t)$\;
    $x^2 \assign \Sample(f_t)$\;
    \leIf{$\calF(x^1) \ge \calF(x^2)$}{$(y^1,y^2) \assign (x^1,x^2)$}{$(y^1,y^2) \assign (x^2,x^1)$}
    $f'_{t+1} \assign f_t + \frac 1 \mu (y^1-y^2)$\;
    $f_{t+1} \assign \minmax(\frac 1n \mathbf{1}_n, f'_{t+1}, (1 - \frac 1n) \mathbf{1}_n)$\;
    $t \assign t+1$\; 
  }
\caption{The compact genetic algorithm (cGA) to maximize a function $\calF : \{0,1\}^n \to \R$.}
\label{alg:cga}
\end{algorithm2e}

\textbf{Well-behaved frequency assumption:} 
For the hypothetical population size $\mu$, we take the common assumption that any two frequencies that can occur in a run of the cGA differ by a multiple of $\frac 1 \mu$. We call this the \emph{well-behaved frequency assumption}. This assumption was implicitly already made in~\cite{HarikLG99} by using even $\mu$ in all experiments (note that the hypothetical population size is denoted by $n$ in~\cite{HarikLG99}). This assumption was made explicit in~\cite{Droste06} by requiring $\mu$ to be even. Both works do not use the frequencies boundaries $\frac 1n$ and $1 - \frac 1n$, so an even value for $\mu$ ensures well-behaved frequencies. 

For the case with frequency boundaries, the well-behaved frequency assumption is equivalent to $(1-\frac 2n)$ being an even multiple of the update step size $\frac 1 \mu$. In this case, $n_\mu = (1 - \frac 2n) \mu \in 2 \N$ and the set of frequencies that can occur is \[F \coloneqq F_\mu \coloneqq \{\tfrac 1n + \tfrac i \mu \mid i \in [0..n_\mu]\}.\] 
This assumption was made, e.g., in the proof of Theorem~2 in~\cite{SudholtW16} and in the paper~\cite{LenglerSW18} (see the paragraph following Lemma~2.1).

\textbf{A trivial lower bound:} We finish this subsection on the cGA with the following very elementary remark, which shows that the cGA with hypothetical population size $\mu$ with probability $1 - \exp(-\Omega(n))$ has a runtime of at least $\min\{\mu/4, \exp(\Theta(n))\}$ on any $\calF : \{0,1\}^n \to \R$ with a unique global optimum. This shows, in particular, that the cGA with the parameter value $\mu = \exp(\Omega(k))$ used to optimize jump functions with gap size $k \in \omega(\log n) \cap o(n)$ in time $\exp(O(k))$ in~\cite{HasenohrlS18} cannot have a runtime better than exponential in $k$.

\begin{lemma}\label{lem:elem}
  Let $\calF : \{0,1\}^n \to \R$ have a unique global optimum. The probability that the cGA generates the optimum of $f$ in $T = \min\{\mu/4, (1.3)^n\}$ iterations is at most $\exp(-\Omega(n))$.  
\end{lemma}

\begin{proof}
  By the definition of the cGA, the frequency vector $f$ used in iteration $t = 1, 2, 3, \dots$ satisfies $f \in [\frac 12 - \frac{t-1}{\mu}, \frac 12 + \frac{t-1}{\mu}]^n$. Consequently, the probability that a fixed one of the two search points which are generated in this iteration is the optimum, is at most $(\frac 12 + \frac{t-1}{\mu})^n$. For $t \le \mu/4$, this is at most $(3/4)^n$. Hence by a simple union bound, the probability that the optimum is generated in the first $T = \min\{\mu/4, (1.3)^n\}$ iterations, is at most $2 T (3/4)^n = \exp(-\Omega(n))$.
\end{proof}

\subsection{Related Work}\label{sec:jump}

In all results described in this section, we shall assume that the hypothetical population size is at most polynomial in the problem size $n$, that is, that there is a constant $c$ such that $\mu \le n^c$. 

The first to conduct a rigorous runtime analysis for the cGA was Droste in his seminal work~\cite{Droste06}. He regarded the cGA without frequency boundaries, that is, he just took $f_{t+1} \coloneqq f'_{t+1}$ in our notation. He showed that this algorithm with $\mu \ge n^{1/2 + \eps}$, $\eps > 0$ any positive constant, finds the optimum of the $\onemax$ function defined by 
\[\onemax(x) = \|x\|_1 = \sum_{i=1}^n x_i\] 
for all $x \in \{0,1\}^n$ with probability at least $1/2$ in $O(\mu \sqrt n)$ iterations~\cite[Theorem~8]{Droste06}. 

Droste also showed that this cGA for any objective function $\calF$ with unique optimum has an expected runtime of $\Omega(\mu \sqrt n)$ when conditioning on no premature convergence~\cite[Theorem~6]{Droste06}. It is easy to see that his proof of the lower bound can be extended to the cGA with frequency boundaries, that is, to Algorithm~\ref{alg:cga}. For this, it suffices to deduce from his drift argument the result that the first time $T_{n/4}$ that the frequency distance $D = \sum_{i=1}^n (1 - f_{it})$ is less than $n/4$ satisfies $E[T_{n/4}] \ge \mu \sqrt n \frac{\sqrt 2}{4}$. Since the probability to sample the optimum from a frequency distance of at least $n/4$ is at most 
\begin{align*}
\prod_{i=1}^n f_{it} &= \prod_{i=1}^n (1 - (1 - f_{it})) \le \prod_{i=1}^n \exp(-(1 - f_{it})) \\
&= \exp\left(-\sum_{i=1}^n (1-f_{it})\right) \le \exp(-n/4),
\end{align*} 
the algorithm with high probability does not find the optimum before time~$T_{n/4}$.

Ten years after Droste's work, Sudholt and Witt~\cite{SudholtW16} showed that the $O(\mu \sqrt n)$ upper bound also holds for the cGA with frequency boundaries. There (but the same should be true for the cGA without boundaries) a hypothetical population size of $\mu =\Omega(\sqrt n \log n)$ suffices (recall that Droste required $\mu = \Omega(n^{1/2+\eps})$). The technically biggest progress with respect to upper bounds most likely lies in the fact that the analysis in~\cite{SudholtW16} also holds for the expected optimization time, which means that it also includes the rare case that frequencies reach the lower boundary. We refer to Section~2.3 of~\cite{Doerr19} for a detailed discussion of the relation of expectations and tail bounds for runtimes of EDAs, including a method to transform EDAs with with-high-probability guarantees into those with guarantees on the expected runtime). Sudholt and Witt also show that the cGA with frequency boundaries with high probability (and thus also in expectation) needs at least $\Omega(\mu\sqrt n + n \log n)$ iterations to optimize $\onemax$. While the $\mu\sqrt n$ lower bound could have been also obtained with methods similar to Droste's, the innocent-looking $\Omega(n \log n)$ bound is surprisingly difficult to prove.

Not much is known for hypothetical population sizes below the order of $\sqrt n$. It is clear that then the frequencies will reach the lower boundary of the frequency range, so working with a non-trivial lower boundary like~$\frac 1n$ is necessary to prevent premature convergence. The recent lower bound $\Omega(\mu^{1/3} n)$ valid for $\mu = O(\frac{\sqrt n}{\log n \log\log n})$ of~\cite{LenglerSW18} indicates that already a little below the $\sqrt n$ regime significantly larger runtimes occur, but with no upper bounds this regime remains largely not understood.

We refer the reader to the recent survey~\cite{KrejcaW18} for more results on the runtime of the cGA on classic unimodal test functions like \leadingones and \binval. Interestingly, nothing was known for non-unimodal functions before the recent work of Hasen\"ohrl and Sutton~\cite{HasenohrlS18} on jump functions, which we discussed already in the introduction. 

The general topic of lower bounds on runtimes of EDAs remains largely little understood. Apart from the lower bounds for the cGA on \onemax discussed above, the following is known. Krejca and Witt~\cite{KrejcaW17} prove a lower bound for the UMDA on \onemax, which is of a similar flavor as the lower bound for the cGA of Sudholt and Witt~\cite{SudholtW16}: For $\lambda = (1 + \beta) \mu$, where $\beta > 0$ is a constant, and $\lambda$ polynomially bounded in $n$, the expected runtime of the UMDA on \onemax is $\Omega(\mu \sqrt n + n \log n)$. For the binary value function \binval, Droste~\cite{Droste06} and Witt~\cite{Witt18} together give a lower bound of $\Omega(\min\{n^2, Kn\})$ for the runtime of the cGA. Apart from these sparse results, we are not aware of any lower bounds for EDAs. Of course, the black-box complexity of the problem is a lower bound for any black-box algorithm, hence also for EDAs, but these bounds are often lower than the true complexity of a given algorithm. For example, the black-box complexities of \onemax, \leadingones, and jump functions with jump size $k \le \frac 12 n - n^{\eps}$, $\eps > 0$ any constant, are $\Theta(\frac{n}{\log n})$~\cite{DrosteJW06,AnilW09}, $\Theta(n \log\log n)$~\cite{AfshaniADDLM13}, and $\Theta(\frac{n}{\log n})$~\cite{BuzdalovDK16}, respectively.

To round off the picture, we briefly describe some typical runtimes of evolutionary algorithms on jump functions. We recall that the $n$-dimensional jump function with jump size $k \ge 1$ is defined by
\[
\jump_{nk}(x) = 
\begin{cases}
\|x\|_1+k & \mbox{if $\|x\|_1 \in [0..n-k] \cup \{n\}$,}\\
n - \|x\|_1 & \mbox{if $\|x\|_1 \in [n-k+1\, ..\, n-1]$}.
\end{cases}
\]
Hence for $k = 1$, we have a fitness landscape isomorphic to the one of $\onemax$, but for larger values of $k$ there is a fitness valley (``gap'')
\[G_{nk} \coloneqq \{x \in \{0,1\}^n \mid n-k < \|x\|_1 < n\}\] 
consisting of the $k-1$ highest sub-optimal fitness levels of the \onemax function. This valley is hard to cross for evolutionary algorithms using standard-bit mutation with mutation rate $\frac 1n$ since with very high probability they need to generate the optimum from one of the local optima, which in a single application of the mutation operator happens only with probability less than $n^{-k}$. For this reason, e.g., the classic $(\mu+\lambda)$ and $(\mu,\lambda)$ EAs all have a runtime of at least $n^k$. This was proven formally for the \oea in the classic paper~\cite{DrosteJW02}, but the argument just given proves the $n^k$ lower bound equally well for all $(\mu+\lambda)$ and $(\mu,\lambda)$ EAs. By using larger mutation rates or a heavy-tailed mutation operator, a $k^{\Theta(k)}$ runtime improvement can be obtained~\cite{DoerrLMN17}, but the runtime remains $\Omega(n^k)$ for $k$ constant. 

Asymptotically better runtimes can be achieved when using crossover, though this is harder than expected. The first work in this direction~\cite{JansenW02}, among other results, could show that a simple $(\mu+1)$ genetic algorithm using uniform crossover with rate $p_c = O(1 / kn)$ obtains an $O(\mu n^2 k^3 + 2^{2k} p_c^{-1})$ runtime when the population size is at least $\mu = \Omega(k \log n)$. A shortcoming of this result, already noted by the authors, is that it only applies to uncommonly small crossover rates. Using a different algorithm that first always applies crossover and then mutation, a runtime of $O(n^{k-1} \log n)$ was achieved by Dang et al.~\cite[Theorem~2]{DangFKKLOSS18}. For $k \ge 3$, the logarithmic factor in the runtime can be removed by using a higher mutation rate. With additional diversity mechanisms, the runtime can be further reduced up to $O(n \log n + 4^k)$, see~\cite{DangFKKLOSS16}. In the light of this last result, the insight stemming from the previous work~\cite{HasenohrlS18} and ours is that the cGA apparently without further modifications supplies the necessary diversity to obtain a runtime of $O(n \log n + 2^{O(k)})$.

Finally, we note that runtimes of $O(n \binom{n}{k})$ and $O(k \log(n) \binom{n}{k})$ were shown for the $(1+1)$~IA$^{\mathrm hyp}$ and the $(1+1)$ Fast-IA artificial immune systems, respectively~\cite{CorusOY17,CorusOY18fast}.

\subsection{Preliminaries}

We now collect a few elementary tools that will be used on our analysis. The first is well-known and the next two are from~\cite{Doerr19}, so it is only the last one which we could not find in the literature.
 
The following estimate seems well-known (e.g., it was used in~\cite{JansenJW05} without proof or reference). Gie{\ss}en and Witt~\cite[Lemma~3]{GiessenW17} give a proof via estimates of binomial coefficients and the binomial identity. A more elementary proof can be found in~\cite[Lemma~10.37]{Doerr18bookchapter}.

\begin{lemma}\label{lprobbino}
  Let $X \sim \Bin(n,p)$. Let $k \in [0..n]$. Then \[\Pr[X \ge k] \le \binom{n}{k} p^k.\]
\end{lemma}

The next estimate was essentially proven in the extended version of~\cite[Lemma~1]{Doerr19}. The only small difference here is that in the first inequality, we used a slightly different Chernoff bound, which also allows deviation parameters $\Delta$ which are greater than one. We therefore omit the proof.

\begin{lemma}\label{lsample}
  Let $f \in [0,1]^n$, $D \coloneqq n - \|f\|_1$, $D^- \le D \le D^+$, $x \sim \Sample(f)$, and $d(x) \coloneqq n - \|x\|_1$. Then for all $\Delta \ge 0$ and $\delta \in [0,1]$, we have
  \begin{align*}
  \Pr[d(x) \ge (1+\Delta) D^+] & \le \exp(-\tfrac 13 \min\{\Delta^2,\Delta\} D^+),\\
  \Pr[d(x) \le (1-\delta) D^-] & \le \exp(-\tfrac 12 \delta^2 D^-).
  \end{align*}
\end{lemma}


To estimate the influence from capping the frequencies into the interval $[\frac 1n, 1 - \frac 1n]$, the following elementary result was shown in the extended version of~\cite{Doerr19}.

\begin{lemma}\label{lboundary}
  Let $P = 2 \frac 1n (1-\frac 1n)$. Let $t \ge 0$. Using the notation given in Algorithm~\ref{alg:cga}, consider iteration $t+1$ of a run of the cGA started with a fixed frequency vector $f_t \in [\frac 1n, 1-\frac 1n]^n$. 
  \begin{enumerate} 
  \item\label{it:boundaryL} Let $L = \{i \in [1..n] \mid f_{it} = \frac 1n\}$, $\ell = |L|$, and $M = \{i \in L \mid x^1_i \neq x^2_i\}$. Then $|M| \sim \Bin(\ell,P)$ and 
  \[\|f_{t+1}\|_1 - \|f'_{t+1}\|_1 \preceq \|(f_{t+1})_{|L}\|_1 - \|(f'_{t+1})_{|L}\|_1 \preceq \tfrac 1\mu |M| \preceq \tfrac 1\mu \Bin(n,\tfrac 2n).\] 
  \item\label{it:boundaryU} Let $L = \{i \in [1..n] \mid f_{it} = 1 - \frac 1n\}$, $\ell = |L|$, and $M = \{i \in L \mid x^1_i \neq x^2_i\}$. Then $|M| \sim \Bin(\ell,P)$ and \[\|f'_{t+1}\|_1 - \|f_{t+1}\|_1 \preceq \|(f'_{t+1})_{|L}\|_1 - \|(f_{t+1})_{|L}\|_1 \preceq \tfrac 1\mu |M| \preceq \tfrac 1\mu \Bin(n,\tfrac 2n).\]
  \end{enumerate} 
\end{lemma}

%
%

To argue that the cGA makes at least some small progress, we shall use the following blunt estimate for the probability that two bit strings $x, y \sim \Sample(f)$ sampled from the same product distribution have a different distance from the all-ones string (and, by symmetry, from any other string, but this is a statement which we do not need here).

\begin{lemma}\label{ldiff}
  Let $n \in \N$, $m \in [\frac n2..n]$, and $f \in [\frac 1n,1-\frac1n]^m$. Let $x^1,x^2 \sim \Sample(f)$ be independent. Then $\Pr[\|x^1\|_1 \neq \|x^2\|_1] \ge \frac 1 {16}$.
\end{lemma}

\begin{proof}
  For all $v \in \R^m$ and $a, b \in [1..m]$ with $a \le b$ we use the abbreviation $v_{[a..b]} \coloneqq \sum_{i=a}^b v_i$. By symmetry, we can assume that $f_{[1..m]} \le \frac m2$. Without loss of generality, we may further assume that $f_i \le f_{i+1}$ for all $i \in [1..m-1]$. We have $f_{\lfloor m/4 \rfloor} \le \frac 23$ as otherwise 
  \[f_{[1..m]} \ge f_{[\lfloor m/4 \rfloor+1..n]} > \tfrac 23  (n - \lfloor m/4 \rfloor) \ge \tfrac 23 \cdot \tfrac 34 m = \tfrac m2,\]
  contradicting our assumption. 
  
  Let $\ell$ be minimal such that $S = f_{[1..\ell]} \ge \frac 18$. Since $\ell \le \frac n8 \le \frac m4$, we have $f_\ell \le \frac 23$ and thus $S \le \frac 18 + \frac 23 = \frac{19}{24}$.
  
  For $j \in \{0,1\}$ let $q_j = \Pr[x^1_{[1..\ell]}=j] = \Pr[x^2_{[1..\ell]}=j]$. We compute 
  \begin{align*}
  \Pr[\|x^1\|_1 \neq \|x^2\|_1] 
  & \ge \Pr[x^1_{[1..\ell]} = x^2_{[1..\ell]} \wedge x^1_{[\ell+1..n]} \neq x^2_{[\ell+1..n]}]\\
  & \quad + \Pr[x^1_{[1..\ell]} \neq x^2_{[1..\ell]} \wedge x^1_{[\ell+1..n]} = x^2_{[\ell+1..n]}]\\
  & = \Pr[x^1_{[1..\ell]} = x^2_{[1..\ell]}] \Pr[x^1_{[\ell+1..n]} \neq x^2_{[\ell+1..n]}]\\
  & \quad + \Pr[x^1_{[1..\ell]} \neq x^2_{[1..\ell]}] \Pr[x^1_{[\ell+1..n]} = x^2_{[\ell+1..n]}]\\
  & \ge \min\{\Pr[x^1_{[1..\ell]} = x^2_{[1..\ell]}],\Pr[x^1_{[1..\ell]} \neq x^2_{[1..\ell]}]\}\\
  & \ge \min\{q_0^2 + q_1^2, 2q_0 q_1\} = 2 q_0 q_1,
%
  \end{align*}
  the latter by the inequality of the arithmetic and geometric mean. Using Bernoulli's inequality, we estimate coarsely   
  \begin{align*}
  q_0 &= \prod_{i=1}^\ell (1-f_i) \ge 1 - f_{[1..\ell]}, \\
  q_1 &= \sum_{i=1}^\ell f_i \prod_{j \in [1..\ell] \setminus \{i\}} (1-f_i) \ge f_{[1..\ell]} (1 - f_{[1..\ell]}).
  \end{align*}
  From the concavity of $z \mapsto z(1-z)^2$ in $[0,1]$, we obtain 
  \begin{align*}
  2 q_0 q_1 &\ge 2 \min\{z(1-z)^2 \mid z \in [\tfrac 18, \tfrac{19}{24}]\} \\
  &= 2 \min\{z(1-z)^2 \mid z \in \{\tfrac 18, \tfrac{19}{24}\}\} = 2 \tfrac {19}{24} (\tfrac 5{25})^2 \ge \tfrac 1 {16}.
  \end{align*}
\end{proof}

\section{Main Result}

We now state precisely our main result, explain the central proof ideas, and state the formal proof. 

\begin{theorem}\label{thm:main}
  There are constants $\alpha_1, \alpha_2 > 0$ such that for any $n$ sufficiently large and any $k \in [1..n]$, regardless of the hypothetical population size $\mu$, the runtime of the cGA on $\jump_{nk}$ with probability $1 - \exp(-\alpha_1 k)$ is at least $\exp(\alpha_2 k)$. In particular, the expected runtime is exponential in $k$.
\end{theorem}

To prove this result, we will regard the stochastic process $D_t \coloneqq n - \|f_t\|_1$, that is, the difference of the sum of the frequencies from the ideal value $n$. Our general argument is that this process with probability $1 - \exp(-\Omega(k))$ stays above $\frac 14 k$ for $\exp(\Omega(k))$ iterations. In each iteration where $D_t \ge \frac 14 k$, the probability that the optimum is sampled, is only $\exp(-\Omega(k))$. Hence there is a $T = \exp(\Omega(k))$ such that with probability $1-\exp(-\Omega(k))$, the optimum is not sampled in the first $T$ iterations.

The heart of the proof is an analysis of the process $(D_t)$. It is intuitively clear that once the process is below $k$, then often the two search points sampled in one iteration both lie in the gap region, which gives a positive drift (that is, a decrease of the average frequency). To turn this drift away from the target (a small $D_t$ value) into an exponential lower bound for the runtime, we consider the process $Y_t = \exp(c \min\{\tfrac 12 k - D_t, \tfrac 14 k\})$, that is, an exponential rescaling of $D_t$. Such a rescaling has recently also been used in~\cite{AntipovDY19}. We note that the usual way to prove exponential lower bounds is the negative drift theorem of Oliveto and Witt~\cite{OlivetoW12}. We did not immediately see how to use it for our purposes, though, since in our process we do not have very strong bounds on the one-step differences. E.g., when $D_t = \frac 12 k$, then the underlying frequency vector may be such that $D_{t+1} \ge D_t + \sqrt k$ happens with constant probability. 

The main technical work now is showing that the process $Y_t$ has at most a constant drift, more precisely, that $E[Y_{t+1} - Y_t] \le 2$ whenever $Y_t < Y_{\max}$. The difficulty is hidden in a small detail. When $D_t \in [\frac 14 k, \frac 34 k]$, and this is the most interesting case, then we have $\|f'_{t+1}\|_1 \ge \|f_t\|$ whenever the two search points sampled lie in the gap region, and hence with probability $1 - \exp(-\Omega(k))$; from Lemma~\ref{ldiff} we obtain, in addition, a true increase, that is, $\|f'_{t+1}\|_1 \ge \|f_t\| + \frac 1 \mu$, with constant probability. Hence the true difficulty arises from the capping of the frequencies into the interval $[\frac 1n,1-\frac 1n]$. This appears to be a minor problem, among others, because only a capping at the lower bound $\frac 1n$ can have an adverse effect on our process, and there are at most $O(k)$ frequencies sufficiently close to the lower boundary. Things become difficult due to the exponential scaling, which can let rare event still have a significant influence on the expected change of the process.   

We now make these arguments precise and prove Theorem~\ref{thm:main}.

\begin{proof}
  Since we are aiming at an asymptotic statement, we can assume in the following that $n$ is sufficiently large.
  
  Since it will ease the presentation when we can assume that $k \ge w(n)$ for some function $w : \N \to \N$ with $\lim_{n \to \infty} w(n) = \infty$, let us first give a basic argument for the case of small $k$. 
  
  We first note that with probability $f_{it}^2 + (1 - f_{it})^2 \ge \frac 12$, the two search points $x^1$ and $x^2$ generated in the $t$-th iteration agree on the $i$-th bit, which in particular implies that $f_{i,t+1} = f_{it}$. Hence with probability at least $2^{-T}$, this happens for the first $T$ iterations, and thus $f_{it} = \frac 12$ for all $t \in [0..T]$. Let us call such a bit position $i$ \emph{sleepy}. 
  
  Note that the events of being sleepy are independent for all $i \in [1..n]$. Hence, taking $T = \lfloor \frac 12 \log_2 n \rfloor$, we see that the number $X$ of sleepy positions has an expectation of $E[X] \ge n 2^{-T} \ge \sqrt n$, and by a simple Chernoff bound, we have $\Pr[X \ge \frac 12 \sqrt n] \ge 1 - \exp(-\Omega(\sqrt n))$. 
  
  Conditional on having at least $\tfrac 12 \sqrt n$ sleepy bit positions, the probability that a particular search point sampled in the first $T$ iterations is the optimum is at most $2^{-\frac 12 \sqrt n}$. By a simple union bound argument, the probability that at least one of the search points generated in the first $T$ iterations is the optimum is at most $2 T 2^{-\frac 12 \sqrt n} = \exp(-\Omega(\sqrt n))$. In summary, we have that with probability at least $1 - \exp(-\Omega(\sqrt n))$, the runtime of the cGA on any function with unique optimum (and in particular any jump function) is greater than $T = \frac 12 \log_2 n$. This implies the claim of this theorem for any $k \le C \log \log n$, where $C$ is a sufficiently small constant, and, as discussed above, $n$ is sufficiently large.
  
  With this, we can now safely assume that $k = \omega(1)$. For the case that $k \ge \frac n {320}$, we will need slightly modified calculations. To keep this proof readable, we hide (but treat nevertheless) this case as follows. We consider a run of the cGA on a jump function $\jump_{nk'}$ with $k' \in [1..n] \cap \omega(1)$ arbitrary and we let $k \coloneqq \min\{k', \frac n {320}\}$. 
  
  Let $X_t \coloneqq \|f_t\|_1 = \sum_{i=1}^n f_{it}$ be the sum of the frequencies at the end of iteration~$t$. Since we are mostly interested in the case where $X_t$ is close to the maximum value, we also define $D_t = n - X_t$. 
  
  Our intuition (which will be made precise) is that the process $(D_t)$ finds it hard to go significantly below $k$ because there we will typically sample individuals in the gap, which lead to a decrease of the sum of frequencies (when the two individuals have different distances from the optimum). To obtain an exponential lower bound on the runtime, we suitably rescale the process by defining, for a sufficiently small constant $c$, 
  \[Y_t = \min\{\exp(c(\tfrac 12 k - D_t)),\exp(\tfrac 14 c k)\} = \exp(c \min\{\tfrac 12 k - D_t, \tfrac 14 k\}).\] 
  Observe that $Y_t$ attains its maximal value $\Ymax = \exp(\frac 14 c k)$ precisely when $D_t \le \frac 14k$. Also, $Y_t \le 1$ for $D_t \ge \frac 12 k$.
  
  To argue that we have $D_t > \tfrac 14 k$ for a long time, we now show that the drift $E[Y_{t+1} - Y_t \mid Y_t < \exp(\frac 14 ck)]$ is at most constant. To this aim, we condition on a fixed value of $f_{t}$, which also determines $X_t$ and $D_t$. We treat separately the two cases that $D_t \ge \frac 34 k$ and that $\frac 34 k > D_t > \frac 14 k$.
   
  \textbf{Case 1:} Assume first that $D_t \ge \frac 34 k$. By Lemma~\ref{lsample}, with probability $1 - \exp(-\Omega(k))$, the two search points $x^1,x^2$ sampled in iteration $t+1$ both satisfy $|\|x^i\|_1 - X_t| < \tfrac 16 (D_t - \tfrac 12 k)$. Let us call this event $A$. In this case, we argue as follows. Let $\{y^1,y^2\} = \{x^1,x^2\}$ such that $f(y^1) \ge f(y^2)$. Then 
  \begin{align*}
  \|f'_{t+1}\|_1 &= \|f_t + \tfrac 1\mu (y^1-y^2)\|_1 \le \|f_t\|_1 + \tfrac 1\mu 2 \tfrac 16 (D_t - \tfrac 12 k) \\
  &\le \|f_t\|_1 + 2 \tfrac 16 (D_t - \tfrac 12 k) = n - D_t + 2 \tfrac 16 (D_t - \tfrac 12 k) \\
  &= n - \tfrac 23 D_t - \tfrac 16k \le n - \tfrac 23 \cdot \tfrac 34 k - \tfrac 16k \le n - \tfrac 23 k.
  \end{align*} 
  
  We still need to consider the possibility that $f_{i,t+1} > f'_{i,t+1}$ for some $i \in [1..n]$. By Lemma~\ref{lboundary}, not conditioning on $A$, we have that $\|f_{t+1}\|_1 - \|f'_{t+1}\|_1 \preceq \frac 1\mu \Bin(\ell,P) \preceq \Bin(\ell,P)$ for some $\ell \in [1..n]$ and $P = 2 \frac 1n (1-\frac 1n)$.
  
   Let us call $B$ the event that $\|f_{t+1}\|_1 - \|f'_{t+1}\|_1 < \frac 16 k$. Note that $A \cap B$ implies $\|f_{t+1}\|_1 < n- \frac 12 k$ and thus $Y_{t+1} \le 1$. By Lemma~\ref{lprobbino} and the estimate $\binom{a}{b} \le (\frac{ea}{b})^b$, we have $\Pr[B \ge \frac 16 k] \le \binom{\ell}{\frac 16 k} P^{k/6} \le (\frac{12e\ell}{kn})^{k/6} \le k^{-\Omega(k)}$. 
   
  We conclude that the event $A \cap B$ holds with probability $1 - \exp(-\Omega(k))$; in this case $Y_t, Y_{t+1} \le 1$. In all other cases,  we bluntly estimate $Y_{t+1} - Y_t \le \Ymax$. This gives $E[Y_{t+1} - Y_t] \le (1 - \exp(-\Omega(k))) \cdot 1 + \exp(-\Omega(k)) \Ymax$. By choosing the constant $c$ sufficiently small and taking $n$ sufficiently large, we have $E[Y_{t+1} - Y_t] \le 2$.

  \textbf{Case 2:} Assume now that $\frac 34 k > D_t > \frac 14k$. Let $x^1,x^2$ be the two search points sampled in iteration $t+1$. By Lemma~\ref{lsample} again, we have $k > n - \|x^i\|_1 > 0$ with probability $1 - \exp(-\Omega(k))$ for both $i \in \{1,2\}$. Let us call this event $A$. Note that if $A$ holds, then both offspring lie in the gap region. Consequently, $\|y^1\|_1 \le \|y^2\|_1$ and thus $\|f'_{t+1}\|_1 \le \|f_t\|_1$.
    
    Let $L = \{i \in [1..n] \mid f_{it} = \frac 1n\}$, $\ell = |L|$, and $M = \{i \in L \mid x^1_i \neq x^2_i\}$ as in Lemma~\ref{lboundary}. Note that by definition, $D_t \ge (1-\frac 1n) \ell$, hence from $D_t < \frac 34 k$ and $n \ge 4$ we obtain $\ell < k$. 
   
   Let $B_0$ be the event that $|M|=0$, that is, $x^1_{|L} =  x^2_{|L}$. Note that in this case, $\|f_{t+1}\|_1 = \|f'_{t+1}\|_1 \le \|f_t + \frac 1 \mu (y^1 - y^2)\|_1$. By Lemma~\ref{lboundary}, Bernoulli's inequality, and $\ell \le k$, we have 
   \[\Pr[B_0] = (1 - 2\tfrac 1n (1-\tfrac 1n))^{\ell} \ge 1 - \tfrac{2\ell}{n} \ge 1 - \tfrac{2k}n.\] 
   Since $\ell < k \le \frac{n}{320} < \frac n2$, by Lemma~\ref{ldiff}, we have $x^1_{|[n]\setminus L} \neq x^2_{|[n]\setminus L}$ with probability at least $\frac{1}{16}$. This event, called $C$ in the following, is independent of $B_0$. We have 
   \[\Pr[A \cap B_0 \cap C] \ge \Pr[B_0 \cap C] - \Pr[\overline A] \ge (1-\tfrac{2k}{n})\tfrac{1}{16} - \exp(-\Omega(k)).\] 
   If $A \cap B_0 \cap C$ holds, then $\|f_{t+1}\|_1 \le \|f_t\|_1 - \frac 1 \mu$. If $A \cap B_0 \cap \overline C$ holds, then we still have $\|f_{t+1}\|_1 \le \|f_t\|_1$. 
   
  Let us now, for $j \in [1..\ell]$, denote by $B_j$ the event that $|M|=j$, that is, that $x^1_{|L}$ and $x^2_{|L}$ differ in exactly $j$ bits. By Lemma~\ref{lboundary} again, we have $\Pr[B_j] = \Pr[\Bin(\ell,P) = j]$.
  
  The event $A \cap B_j$ implies $\|f_{t+1}\|_1 \le \|f'_t\|_1 + \frac j \mu  \le \|f_t\|_1 + \frac j \mu$ and occurs with probability $\Pr[A \cap B_j] \le \Pr[B_j] = \Pr[\Bin(\ell,P) = j]$.
  
  Taking these observations together, we compute
  \begin{align}
  E[Y_{t+1}] 
  &= \Pr[\overline A] \, E[Y_{t+1} \mid \overline A] \nonumber\\
  &\quad + \sum_{j=1}^\ell \Pr[A \cap B_j] \, E[Y_{t+1} \mid A \cap B_j] \nonumber\\
  &\quad + \Pr[A \cap B_0 \cap \overline C] \, E[Y_{t+1} \mid A \cap B_0 \cap \overline C] \nonumber\\
  &\quad + \Pr[A \cap B_0 \cap C] \, E[Y_{t+1} \mid A \cap B_0 \cap C] \nonumber\\
  &\le \exp(-\Omega(k)) \Ymax \label{eq:bilanz}\\
  &\quad + \sum_{j=1}^\ell \Pr[\Bin(\ell,P) = j] Y_t \exp(\tfrac{cj}{\mu}) \nonumber\\
  &\quad + \Pr[\Bin(\ell,P) = 0] Y_t\nonumber\\
  &\quad - (\tfrac 1 {16} (1 - \tfrac{2k}{n}) - \exp(-\Omega(k))) Y_t \exp(-\tfrac{c}{\mu}). \nonumber
  \end{align} 
We note that the second and third term amount to $Y_t E[\exp(\frac{cZ}{\mu})]$, where $Z \sim \Bin(\ell,P)$. Writing $Z = \sum_{i=1}^\ell Z_i$ as a sum of $\ell$ independent binary random variables with $\Pr[Z_i = 1] = P$, we obtain 
\[E[\exp(\tfrac{cZ}{\mu})] = \prod_{i=1}^\ell E[\exp(\tfrac{cZ_i}{\mu})] = (1 - P + P\exp(\tfrac{c}{\mu}))^\ell.\]
By assuming $c \le 1$ and using the elementary estimate $e^x \le 1 + 2x$ valid for $x \in [0,1]$, see, e.g., Lemma~4.2(b) in~\cite{Doerr18bookchapter}, we have $1 - P + P\exp(\frac{c}{\mu}) \le 1 + 2P(\frac{c}{\mu})$. Hence with $P \le \frac 2n$, $c \le 1$, $\mu \ge 1$, and $\ell \le \frac n{320}$, we obtain 
\[E[\exp(\tfrac{cZ}{\mu})] \le (1 + 2P(\tfrac{c}{\mu}))^\ell \le \exp(2P(\tfrac{c}{\mu})\ell) \le \exp(\tfrac{4\ell}{n}) \le 1 + \tfrac{8\ell}{n}\]
again by using $e^x \le 1 + 2x$.

In the first term of~\eqref{eq:bilanz}, we again assume that $c$ is sufficiently small to ensure that $\exp(-\Omega(k)) \Ymax = \exp(-\Omega(k)) \exp(\frac 14 ck) \le 1$. Recalling that $k \le \frac{n}{320}$, we finally estimate in the last term 
\[\tfrac 1 {16} (1 - \tfrac{2k}{n}) - \exp(-\Omega(k))) \exp(-\tfrac{c}{\mu}) \ge \tfrac 1{20} \cdot \tfrac 12 = \tfrac 1 {40}.\] 
With these estimates and using, again, $\ell \le k \le \frac n {320}$,
we obtain $E[Y_{t+1}] \le 1 + (1 + \tfrac{8\ell}{n}) Y_t - \tfrac 1 {40} Y_t \le 1 + Y_t$ and thus $E[Y_{t+1} - Y_t] \le 1$.

In summary, we have now shown that the process $Y_t$ satisfies $E[Y_{t+1} - Y_t \mid Y_t < Y_{\max}] \le 2$. We note that $Y_0 \le 1$ with probability one. For the sake of the argument, let us artificially modify the process from the point on when it has reached a state of at least $Y_{\max}$. So we define $(\tilde Y_t)$ by setting $\tilde Y_t = Y_t$, if $Y_t < Y_{\max}$ or if $Y_t \ge Y_{\max}$ and $Y_{t-1} < Y_{\max}$, and $\tilde Y_t = \tilde Y_{t-1}$ otherwise. In other words, $(\tilde Y_t)$ is a copy of $(Y_t)$ until it reaches a state of at least $Y_{\max}$ and then does not move anymore. With this trick, we have $E[\tilde Y_{t+1} - \tilde Y_t] \le 2$ for all $t$. 

A simple induction and the initial condition $\tilde Y_0 \le 1$ shows that $E[\tilde Y_t] \le 2t+1$ for all $t$. In particular, for $T = \exp(\frac 18 c k)$, we have $E[Y_T] \le 2 \exp(\frac 18 c k) + 1$ and, by Markov's inequality, 
\[\Pr[\tilde Y_T \ge Y_{\max}] \le \frac{2 \exp(\frac 18 c k) + 1}{Y_{\max}} = \exp(-(1-o(1)) \tfrac 18 c k).\]
  
Hence with probability $1 - \exp(-(1-o(1)) \tfrac 18 c k)$, we have $\tilde Y_T < Y_{\max}$. We now condition on this event. By construction of $(\tilde Y_t)$, we have $Y_t < Y_{\max}$, equivalently $D_t > \frac 14 k$, for all $t \in [0..T]$. If $D_t > \frac 14 k$, then the probability that a sample generated in this iteration is the optimum, is at most $\prod_{i=1}^n f_{it} = \prod_{i=1}^n (1 - (1 - f_{it})) \le \prod_{i=1}^n \exp(-(1-f_{it})) = \exp(-(n - \|f_t\|_1)) = \exp(-D_t) \le \exp(-\tfrac 14 k)$. Assuming $c \le 1$ again, we see that the probability that the optimum is generated in one of the first $T$ iterations, is at most $2 T \exp(-\frac 14 k) = 2 \exp(\frac 18 c k) \exp(-\frac 14 k) = \exp(-(1-o(1)) \frac 18 k)$. This shows the claim. 
\end{proof}

\section{An $\Omega(n \log n)$ Lower Bound?}\label{sec:nlogn}

With the exponential lower bound proven in the previous section, the runtime of the cGA on jump functions is well understood, except that the innocent looking lower bound $\Omega(n \log n)$, matching the corresponding upper bound for $k < \frac 1 {20} \ln n$, is still missing. Since Sudholt and Witt~\cite{SudholtW16} have proven an $\Omega(n \log n)$ lower bound for the simple unimodal function $\onemax$, which for many EAs is known to be one of the easiest functions with unique global optimum~\cite{DoerrJW12algo,Sudholt13,Witt13,Doerr18evocop}, it would be very surprising if this lower bound would not hold for jump functions as well.

Unfortunately, we do not see any easy way to prove such a lower bound. We strongly believe that the proof of~\cite{SudholtW16} can be extended to also include jump functions, but since this proof is truly complicated, we shy away from taking such an effort to prove a result that would be that little surprising. We instead argue here why the usual ``\onemax is the easiest case'' argument fails. While we would not say that it is not a valuable research goal to extend the proof of~\cite{SudholtW16} to jump functions, we would much prefer if someone could prove a general $\Omega(n \log n)$ lower bound for all functions with unique global optimum (or disprove this statement). 

The true reason why \onemax is the easiest optimization problem for many evolutionary algorithms, implicit in all such proofs and explicit in~\cite{Doerr18evocop}, is that when comparing a run of the evolutionary algorithm on \onemax and some other function $\calF$ with unique global optimum, then at all times the Hamming distance between the current-best solution and the optimum in the \onemax process is stochastically dominated by the one of the other process. This follows by induction and a coupling argument from the following key insight (here formulated for the \oea only).

\begin{lemma}\label{lem:dom}
  Let $\calF : \{0,1\}^n \to \R$ be some function with unique global optimum $x^*$ and let \onemax be the $n$-dimensional \onemax function with unique global optimum $y^* = (1, \dots, 1)$. Let $x,y \in \{0,1\}^n$ such that $H(x,x^*) \ge H(y,y^*)$, where $H(\cdot, \cdot)$ denotes the Hamming distance. Consider one iteration of the \oea optimizing $\calF$, started with $x$ as parent individual, and denote by $x'$ the parent in the next iteration. Define $y'$ analogously for \onemax and $y$. Then $H(x',x^*) \succeq H(y',y^*)$.
\end{lemma}

As a side remark, note that the lemma applied in the special case $\calF = \onemax$ shows that the intuitive rule ``the closer a search point is to the optimum, the shorter is the optimization time when starting from this search point'' holds for optimizing \onemax via the \oea.

We now show that a statement like Lemma~\ref{lem:dom} is not true for the cGA. Since the states of a run of the cGA are the frequency vectors $f$, the natural extension of the Hamming distance quality measure above is the $\ell_1$-distance $d(f,x^*) = \|f-x^*\|_1 = \sum_{i=1}^n |f_i - x^*_i|$. 

Consider now a run of the cGA on an $n$-dimensional ($n$ even for simplicity) jump function $\calF$ with jump size $k \le n/4$. Consider one iteration starting with the frequency vector $f = \frac 12 \textbf{1}_n$. For comparison, consider one iteration of the cGA optimizing \onemax, starting with a frequency vector $g \in [0,1]^n$ such that half the entries of $g$ are equal to $\frac 1n + \frac 1\mu$ and the other half equals $1 - \frac 1n - \frac 1\mu$. Let us take $\mu = n$ for simplicity. Note that both $\calF$ and $\onemax$ have the same unique global optimum $x^* = y^* = (1, \dots, 1)$. 

We obviously have $d(f,x^*) \ge d(g,y^*)$, since both numbers are equal to~$\frac n2$. Let $f', g'$ be the frequency vectors after one iteration. Since with probability $1 - \exp(-\Omega(n))$, both search points sampled in the jump process have between $\frac n4$ and $\frac 3{4n}$ ones, their jump fitnesses equal their \onemax fitnesses. Consequently, we may apply Lemma~5 from~\cite{Droste06} and see that $E[d(f',x^*)] \le \frac n2 - \Omega(\sqrt n)$. For the \onemax process, however, denoting the two search points generated in this iteration by $x^1$ and $x^2$, we see 
\begin{align*}
E[\|g - g'\|_1]^2 &\le E[\|g - g'\|_1^2] = E\left[\left(\sum_{i=1}^n (x^1_i-x^2_i)\right)^2\right] \\
&= \Var\left[\sum_{i=1}^n (x^1_i-x^2_i)\right] = O(1)
\end{align*}
 and hence $E[d(g',y^*)] \ge d(g',y) - O(1) = \frac n2 - O(1)$. Consequently, we cannot have $d(f',x^*) \succeq d(g',y^*)$.

We note that a second imaginable domination result is also not true. Assume, for simplicity, that we optimize a function $\calF$ with unique global maximum equal to $(1,\dots,1)$ and the function \onemax via the cGA with same parameter setting. If $f \le g$ (component-wise), and $f'$ is the frequency vector resulting from one iteration optimizing $\calF$ starting with $f$ and $g'$ is the frequency vector resulting from one iteration optimizing $\onemax$ starting with $g$, then in general we do not have $f'_i \preceq g'_i$ for all $i \in [1..n]$. 

As counter-example, let  $f = (\frac 12, \frac 1n, \dots, \frac 1n)$, but now $g = \frac 12 \textbf{1}_n$. Clearly, $f \le g$. We now consider the results of one iteration of the cGA, always with the \onemax function as objective. When performing one iteration of the cGA on \onemax started with $f$, and denoting the two samples by $x^1$ and $x^2$ and their quality difference in all but the first bit by $\Delta = \|x^1_{|[2..n]}\|_1 - \|x^2_{|[2..n]}\|_1$, then the resulting frequency vector $f'$ satisfies 
\begin{align}
  \Pr[f'_1 = \tfrac 12 + \tfrac 1 \mu] &= \Pr[x^1_1 \neq x^2_1] (\tfrac 12 \Pr[\Delta \notin \{-1,0\}] + \Pr[\Delta \in \{-1,0\}]) \nonumber\\ 
&= \Pr[x^1_1 \neq x^2_1] (\tfrac 12 + \tfrac 12 \Pr[\Delta \in \{-1,0\}]).\label{eq:counterex}
\end{align}
Since $\Pr[\Delta \in \{-1,0\}] \ge \Pr[\|x^1_{|[2..n]}\|_1 = \|x^2_{|[2..n]}\|_1 = 0] = (1 - \frac 1n)^{2(n-1)} \ge 1/e^2$, we have $\Pr[f'_1 = \tfrac 12 + \tfrac 1 \mu] \ge \frac 14 + \frac 1 {4e^2}$.

When starting the iteration with $g$, the resulting frequency vector $g'$ satisfies an equation analoguous to~\eqref{eq:counterex}, but now $\Delta$ is the difference of two binomial distributions with parameters $n-1$ and $\frac 12$. Hence, we have $\Pr[\Delta \in \{-1,0\}] = O(n^{-1/2})$, see, e.g., \cite[Lemma~4.13]{Doerr18bookchapter} for this elementary estimate, and thus $\Pr[f'_1 = \tfrac 12 + \tfrac 1 \mu] = \frac 14 + o(1)$, disproving that $f'_1 \preceq g'_1$.

In summary, the richer mechanism of building a probabilistic model of the search space in the cGA (as opposed to using a population in EAs) makes is hard to argue that \onemax is the easiest function for the cGA. This, in particular, has the consequence that lower bounds for the runtime of the cGA on \onemax cannot be easily extended to other functions with a unique global optimum. 

\section{Conclusion}

The main result of this work is an $\exp(\Omega(k))$ lower bound for the runtime of the cGA on jump functions with jump size $k$, regardless of the hypothetical population size $\mu$. This in particular shows that the result of Hasen\"ohrl and Sutton~\cite{HasenohrlS18} cannot be improved by running the cGA with a hypothetical population size that is sub-exponential in $k$. 

What is noteworthy in our proof is that it does not require a distinction between the different cases that frequencies reach boundary values or not (as in, e.g., the highly technical lower bound proof for \onemax in~\cite{SudholtW16}. It seems to be an interesting direction for future research to find out to what extend such an approach can be used also for other lower bound analyses.

As a side result, we observed that two natural domination arguments that could help showing that \onemax is the easiest function for the cGA are not true. For this reason, the natural lower bound of $\Omega(n \log n)$ remains unproven. Proving it, or even better, proving that $\Omega(n \log n)$ is a lower bound for the runtime of the cGA on any function $\calF : \{0,1\}^n \to \R$ with a unique global optimum, remain challenging open problems.

}


\newcommand{\etalchar}[1]{$^{#1}$}

\end{document}